\documentclass[runningheads]{llncs}
\usepackage{graphicx}
\usepackage{amssymb}
\usepackage{amsmath}
\usepackage{booktabs}
\usepackage{float}
\usepackage{multicol}
\usepackage{multirow}
\usepackage{subfigure}
\usepackage[inkscapelatex=false]{svg}

\newcommand{\topcaption}{%
\setlength{\abovecaptionskip}{0pt}%
\setlength{\belowcaptionskip}{10pt}%
\caption}

\begin{document}
\title{Balanced Graph Structure Learning for Multivariate Time Series Forecasting}
%
%

\author{Weijun Chen \and
Yanze Wang \and
Chengshuo Du \and
Zhenglong Jia \and
Feng Liu\thanks{Corresponding Author} \and
Ran Chen}
\authorrunning{Weijun Chen et al.}
%
\institute{
Beihang University
\\
\email{\{onceCWJ, king.donmn, chengshuodu711, gavinjzl23, liufeng001011, chenran0522\}@gmail.com}}

\maketitle              

\begin{abstract}
Accurate forecasting of multivariate time series is an extensively studied subject in finance, transportation, and computer science. Fully mining the correlation and causation between the variables in a multivariate time series exhibits noticeable results in improving the performance of a time series model. Recently, some models have explored the dependencies between variables through end-to-end graph structure learning without the need for predefined graphs. However, current models do not incorporate the trade-off between efficiency and flexibility and lack the guidance of domain knowledge in the design of graph structure learning algorithms. This paper alleviates
the above issues by proposing Balanced Graph Structure Learning for Forecasting (BGSLF), a novel deep learning model that joins graph structure learning and forecasting. Technically, BGSLF leverages the spatial information into convolutional operations and extracts temporal dynamics using the diffusion convolutional recurrent network. The proposed framework balance the trade-off between efficiency and flexibility by introducing Multi-Graph Generation Network (MGN) and Graph Selection Module. In addition, a method named Smooth Sparse Unit (SSU) is designed to sparse the learned graph structures, which conforms to the sparse spatial correlations in the real world. Extensive experiments on four real-world datasets demonstrate that our model achieves state-of-the-art performances with minor trainable parameters. Code will be made publicly available.

\end{abstract}
\section{Introduction}

Today, our lives benefit significantly from various sensors in many fields, such as weather forecasting, transportation, hydrology, electricity, and many other forms of data. The multivariate time series (MTS) data generated by sensors has high practical value and attracts many scholars to participate in the research. MTS forecasting is vital for a learning system that operates in an evolving environment. There is already some valuable work on this aspect~\cite{GDN,DCRNN,MTGNN}.

A fundamental assumption in MTS forecasting is the correlations between variables, which means that a variable's future information depends not only on its historical information but also on the historical information of other variables. Traditional methods, such as autoregressive integrated moving average (ARIMA)~\cite{ARIMA} and vector auto-regression (VAR)~\cite{VAR}, are used in many time series forecasting tasks.
However, these models are insufficient to mine intricate spatial-temporal dynamics or model nonlinear dependencies between MTS data. Recently, some researchers have shifted to deep learning and concentrated on exploiting prominent temporal patterns shared by MTS, such as TPA-LSTM~\cite{TPA-LSTM} and LSTNet~\cite{LSTNet}. These methods have a strong capability in modeling temporal dynamics but lack the ability to capture dynamic spatial relationships.

In the MTS forecasting tasks, effectively modeling and utilizing the correlations between variables is still a challenging problem. Graph neural networks (GNNs) have shown high capability in handling relational dependencies due to their compositionality, local connectivity, and permutation-invariance, so some early work~\cite{DCRNN,MRA-BGCN} has attempted to introduce them into MTS forecasting. However, these graph neural network methods require a predefined graph structure, and the predefined graph structures generally are local and static. Hence, they ignore the long-range dependencies of some nodes and fail to consider the dynamic property of MTS data. Moreover, we cannot obtain such an underlying graph structure in many cases. The method of graph structure learning (GSL) has been proposed to solve the above questions and attracted much attention~\cite{RobustSurvey}. Graph structure learning aims to learn the optimal graph structure and corresponding representation jointly. Furthermore, some literature~\cite{SLAPS,LDS,GTS} has revealed that joint graph learning and downstream tasks are better than directly using predefined graphs, partly due to the noise in the predefined graph structure.

Recent models apply graph structure learning to MTS forecasting and achieve promising results. These representative models are MTS forecasting with GNNs (MTGNN)~\cite{MTGNN}, Graph for Timeseries (GTS)~\cite{GTS}, Adaptive Graph Convolutional Recurrent Network (AGCRN)~\cite{AGCRN}, Graph WaveNet (GWN)~\cite{GWN}, Spatial-temporal attention wavenet (STAWnet)~\cite{STAWnet} , Graph Deviation Network (GDN) and Neural Relational Inference (NRI)~\cite{NRI}. Despite promising results of joint graph structure learning and forecasting in current models, we argue that these approaches face three major shortcomings.

First, the current models do not consider the trade-off between efficiency and flexibility. The current models either learn a graph adjacency matrix globally (shared by all time series) or infer an adjacency matrix for each batch. The former can be more efficient but less flexible as we cannot adjust the graph for different inputs during testing. On the other hand, the latter enjoy more flexibility but less efficiency as we need to allocate much memory to store the individual adjacency matrices.

Second, current models like GDN, MTGNN, GWN, and AGCRN essentially generate graph adjacency matrices through random initialization and refine the graph structure through end-to-end learning. Although some other models such as GTS and STAWnet apply training sets of MTS for graph inference, they do not apply substantial domain knowledge to fully mine the correlations in multivariate time series, resulting in poor interpretability and easy overfitting.  

Third, some models~\cite{MTGNN,GDN} apply non-differentiable functions to get sparse graph matrices, which are of high gradient variance bringing increased end-to-end training difficulty. 

To emphasize the issues mentioned above, we propose a concise yet practical graph structure learning framework for multivariate time series forecasting. Our model considers the balance between efficiency and flexibility and integrates domain knowledge into the graph structure learning module. Moreover, we propose a new method to obtain sparse and continuous graph matrices. The main contributions of our works are as follows:
\begin{itemize}
\item We propose a model called \emph{Balanced Graph Structure Learning for Forecasting} (BGSLF), which follows a different route with the aim of learning smooth and sparse dependencies between variables while simultaneously training the forecasting module. Different from other models, our model can generate a specified number of graphs through the graph structure learning module to balance efficiency and flexibility. Furthermore, we can select the best graph structure for forecasting during training and testing by measuring the similarity between the time series variables and all graphs.

\item In the graph structure learning module, in order to learn graphs that are more adaptable to prediction, we incorporate some concise yet compelling domain knowledge into MTS forecasting. Through our method, we save lots of parameters and improve forecasting performances. 

\item Inspired by~\cite{SmoothManifolds}, we propose the Smooth Sparse Unit (SSU) intending to infer continuous and sparse dependencies between variables. With the aid of SSU, the use of non-differentiable functions (e.g., Top-$K$ operation) or regularization in inferring sparse graph structures can be avoided.

\item We conduct extensive experiments on four real-world multivariate time series datasets,  PEMS04, PEMS08, METR-LA, Solar-Energy. As a result, the proposed model achieves state-of-the-art results with minor trainable parameters.   

\end{itemize}
 
\section{Related Work}

\subsection{Spatial-temporal Graph Networks}
The graph neural network has achieved great success in capturing spatial relationships. In order to capture this spatial connection, a variety of different methods have been proposed~\cite{GCN,GAT}. Most of these methods essentially follow the neighborhood aggregation strategy, in which the node representation updates itself by iteratively aggregating the representation of the neighbors in the graph. Recently, to solve the complicated spatial and temporal connections in traffic prediction and skeleton-based action recognition, spatial-temporal graph networks were proposed and achieved superior results. The input of spatial-temporal graph networks is usually a multivariate time series and an additionally given adjacency matrix. They aim to predict future values or labels of multivariate time series. The objective of the spatial-temporal graph networks is to make full use of the structural information to achieve the optimal forecasting effect. DCRNN~\cite{DCRNN} uses diffusion convolution and encoder-decoder structures to capture spatial and temporal relationships, respectively.
ASTGCN~\cite{ASTGCN} adopts the spatial-temporal attention mechanism to model dynamic spatial-temporal correlations in traffic data. More recent works such as MRA-BGCN~\cite{MRA-BGCN} and GMAN~\cite{GMAN} further add more complicated spatial and temporal attention mechanisms with GCN~\cite{GCN} to capture the dynamic spatial and temporal correlations. The above methods were successful in traffic prediction at the time. However, these models are limited by the predefined graph structure, so these methods' performance and generalization ability can be improved.

\subsection{Graph Structure Learning}
Graph representation learning is the core of many forecasting tasks, ranging from traffic forecasting to fraud detection. Many graph neural network methods are susceptible to the quality of the graph structure and require a perfect graph structure for learning embeddings. We discuss selected work about multivariate time series forecasting and refer the reader to~\cite{RobustSurvey} for a complete survey. Most of the recent works consider spatial and temporal modules separately. GWN~\cite{GWN} captures the spatial dependency by training an adaptive adjacency matrix. STAWnet~\cite{STAWnet} takes the method of attention to get self-learned node embedding so as to capture the spatial dependency. Both STAWnet and GWN use dilated casual convolutions~\cite{TCN} as the temporal forecasting module. MTGNN learns two embedding vectors per node and obtains the graph adjacency matrix through mathematical transformation. Similar to MTGNN, GDN infers the graph by learning a node embedding per node and builds a $k$NN graph~\cite{KNN} where the similarity metric is the cosine of a pair of embeddings. AGCRN~\cite{AGCRN} proposes two adaptive modules for enhancing graph convolutional networks to infer dynamic spatial dependencies among different traffic series. LDS~\cite{LDS} approximately solves a bilevel programming problem to jointly learn the Bernoulli distribution of the adjacency matrices and the parameters of graph convolutional networks. In order to reduce expensive computation in LDS, GTS no longer regards the graph structure as hyperparameters for optimization but transforms the problem into unilevel programming. Moreover, some studies~\cite{SLAPS,Study} have revealed that there does not exist a golden standard for evaluating the quality of the learned graph structure except for forecasting accuracy.

\begin{figure*}
    \centering
    \includegraphics[width=\textwidth]{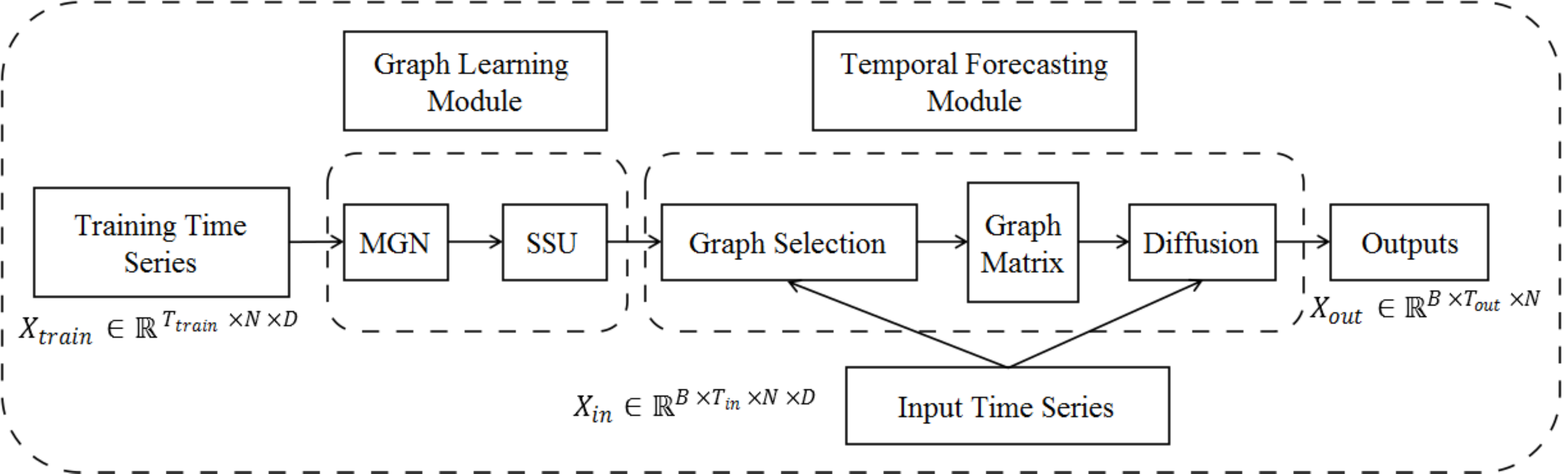}
    \caption{The framework of BGSLF. The training time series are passed through the graph learning module to generate the graph set $\mathbb{A}$, containing two components, Multi-Graph Generation Network (MGN) and Smooth Sparse Unit (SSU). Graph selection is performed in the temporal forecasting module by measuring the similarity between the input time series and the graphs in $\mathbb{A}$. Finally, the selected graph matrix and the input data are fed into the diffusion convolutional network to produce the output.}
    \label{framework}
\end{figure*} 

\section{Methodology}
In this section, we first give a mathematical description of the problem we are addressing in this paper. Next, we describe two building blocks of our framework, the graph structure learning module and the temporal forecasting module. They work together to capture the spatial-temporal dependencies. Finally, we outline the architecture of our framework.

\subsection{Problem Formulation}
In this paper, we focus on exploiting graph structure learning to improve the accuracy of multivariate time series forecasting. Let $x_t \in \mathrm{R}^{N\times D}$ represent the value of a multivariate variable of dimension $N$ at time step $t$, and $D$ denote the feature dimension, where $ x_t[i] \in \mathbb{R}^{D} $ denotes the $i^{th}$ variable at time step $t$. Given the historical $M$ time steps observation sequence of a multivariate variable $\textbf{X} = \{\textbf{x}_{t_1},\textbf{x}_{t_2},\cdots,\textbf{x}_{t_M}\}$,  our goal is to predict the future $N$-step numerical sequence $\textbf{Y} = \{\textbf{x}_{t_{M+1}},\textbf{x}_{t_{M+2}},\cdots,\textbf{x}_{t_{M+N}}\}$. Note that we do not need a predefined graph structure here.
Specifically, let $X_{\text{train}}$ and $X_{\text{valid}}$ denote the training and validation sets of multivariate time series respectively, $A \in \mathbb{R}_{+}^{K\times K}$ is the adjacent matrix of the graph representing the proximity between $K$ time series, $\omega$ denote the parameters used in the GNN and $L$ and $F$ denote the loss functions used during training and validation respectively, the use of graph structure learning for MTS forecasting has a bilevel programming architecture as
\begin{equation}
\begin{aligned}
&\mathop{\text{min}}\limits_{A,\omega_A} \ F(\omega_A,A,X_{\text{valid}}),\\
&\ \text{s.t.} \ \ \ \ \,\omega_A \in \text{arg} \mathop{\text{min}}\limits_{\omega} L(\omega,A,X_{\text{train}}).
\end{aligned}
\end{equation}

Intuitively, the hierarchical relationship results from the fact that the mathematical program related to the parameters of graph structure learning is part of the constraint of the temporal forecasting module. However, the bilevel program problem is naturally difficult to solve. Even for the simplest example, the linear-linear bilevel programming is proved to be NP difficult~\cite{OPT}. Therefore, we need to make some approximations to the original bilevel problem. Similar to~\cite{GTS}, we consider approximating the bilevel programming to a unilevel programming problem as
\begin{equation}
 \mathop{\text{min}}\limits_{A(w)} \ F(w,A,X_{\text{train}}).
\end{equation}

Because this approach owns the freedom to design the parameterization and can better control the number of parameters compared to an inner optimization $w_A$. Therefore, the design of a reasonable parameterization approach is crucial for the graph structure learning module.

\subsection{Graph Structure Learning Module}
Graph Structure Learning is an essential operation when graph structure is missing or incomplete. LDS gives the first mathematical description of bilevel programming applying graph structure learning to downstream tasks. GTS approximates the bilevel optimization to the unilevel optimization. Both apply Bernoulli sampling to generate discrete adjacency matrix $A\in\{0,1\}^{K\times K}$. However, binary values can not represent the rich correlations between variables. Therefore, in our model, instead of constructing discrete adjacency matrices based on Bernoulli distribution, we construct continuous adjacency matrices with each entry $A_{i,j}\in [0,1]$.
The training set of MTS is represented by $X_{\text{train}} \in \mathbb{R}^{T_{\text{train}} \times N \times D} $ where $T_{\text{train}}$ denotes the number of training time steps, $N$ is the number of variables and $D$ denotes the feature dimension. Given a graph adjacency matrix $A$ and its historical $W$ step graph signals, our problem is to learn a function
$\mathcal{F}_A$ which is able to forecast its next $H$ step graph signals. The overall forecasting function can be written as
\begin{align}
&[\mathbf{X}_{t-W+1:t},A]\xrightarrow{\mathcal{F}_A}\mathbf{X}_{t+1:t+H}.
\end{align}

\subsubsection{Multi-Graph Generation Network}
In the graph structure learning module, our purpose is to extract dynamic spatial relationships between variables from the training MTS. Changes in the values of different variables at cross-time can better reflect the spatial relationships between variables. In the transportation domain, for example, the numerical changes of the sensors over time offer insights into how traffic dynamics propagate along with the network. Therefore, we first do the difference operation on the training MTS in order to reveal more moderate correlations:
\begin{equation}
\begin{aligned}
\mathcal{D}iff(\mathbf{X}_{:,1},\mathbf{X}_{:,2},\mathbf{X}_{:,3},\cdots,\mathbf{X}_{:,T})&=
\{\mathbf{X}_{:,2}-\mathbf{X}_{:,1},\mathbf{X}_{:,3}-\mathbf{X}_{:,2},\cdots,\mathbf{X}_{:,T}-\mathbf{X}_{:,T-1}\},\\
&  
\triangleq \{\mathbf{\hat{X}}_{:,1},\mathbf{\hat{X}}_{:,2}, ...,\mathbf{\hat{X}}_{:,T-1}\}.
\end{aligned}
\end{equation}

Then, considering the periodicity of the time series, we set a hyper-parameter period $ P $ to segment the training MTS into $ S = \lfloor T_{\text{train}}/P \rfloor$ segments, each containing time series $\mathbf{\hat{X}}_i \in \mathbb{R}^{N \times{D} \times {P}}, i = 1,2,...,S. $
After obtaining the time-series segments, we concatenate these segments to obtain a four-dimensional tensor $\mathcal{O} = [\mathbf{\hat{X}_1},\mathbf{\hat{X}_2}......\mathbf{\hat{X}_S}]\in \mathbb{R}^{S\times N \times D \times P}$. 
Subsequently, we use 2D convolution and two fully connected layers to transform the four-dimensional tensor $\mathcal{O}$ to obtain $R$ graphs. The number of input channels is $S$, and the number of output channels is the number of graphs $R$ we aim to get. These graphs constitute the graph set $\mathbb{A}$.

\subsubsection{Smooth Sparse Unit}
In this section, we propose the Smooth Sparse Unit (SSU) to learn continuous and sparse graphs instead of using the discrete adjacency matrix produced by Gumbel-softmax sampling~\cite{GumbelSoftmax} in GTS. Inspired by Lee~\cite{SmoothManifolds}, the mathematical principles of SSU are as follows:
\begin{lemma}\label{lemma1}
\cite{SmoothManifolds} The function $f : \mathbb{R} \to \mathbb{R}$ defined by
\begin{equation}\label{f}
 f(x)=\left\{
\begin{aligned}
&e^{-\frac{1}{x}} &(x>0), \\
&0 &(x \leq 0),
\end{aligned}
\right.
\end{equation}
is smooth.
\end{lemma}

\begin{lemma}
There exists a smooth function $\varphi : \mathbb{R} \to [0,1]$ such that $\varphi(x) \equiv 0 \ \text{for}\  x \leq 0;\varphi(x) \in (0,1)\ \text{for}\ 0<x<1;\varphi(x) \equiv 1 \ \text{for}\ x \geq 1$.
\end{lemma}
\begin{proof}
Let
\begin{equation}\label{phi}
\varphi(x)=\dfrac{\alpha f(x)}{\alpha f(x)+f(1-x)}\ (\alpha \in \mathbb{R}_+),
\end{equation}
where $f \in C^{\infty}(\mathbb{R})$ is defined by Equation~\eqref{f}. 
It is easy to check that $\varphi(x) \in [0,1]$,  and $\varphi(x)$ is 0 for $x \leq 0$, 1 for $x \geq 1$.
\end{proof}

Using the above mathematical formula, the output adjacency matrix $A$ is $$A = \dfrac{\alpha f(G)}{\alpha f(G)+f(\mathbf{1}-G)},$$
where $G \in \mathbb{R}^{N \times N}$ is the output of the fully connected layers, $\mathbf{1}$ denotes the all one matrix, $\alpha$ is the sparsification coefficient (hyper-parameter), $f$ is an element-wise operator defined by Lemma~\ref{lemma1}, and $A$ is the graph that we finally learned.

More details for controlling the sparsity effect of $\alpha$ and techniques of redefining gradients to accelerate convergence are presented in Appendix~\ref{effect}.

\subsection{Temporal Forecasting module}

\subsubsection{Graph Selection}
After obtaining $R$ graphs through the MGN and SSU, the optimal graph structure should be selected for each input time series $X_{\text{in}} \in \mathbb{R}^{B \times T_{\text{in}} \times N \times D}$, and we use the following objective function to represent:
\begin{equation}
    A=\underset{A_i\in \mathbb{A}}{arg\max}\,\cos{\left< \mathcal{X}^T\mathcal{X},A_i \right>},
\end{equation}
where $\cos{\left< \mathcal{X},\mathcal{Y} \right>}=\frac{\sum\limits_{i,j}{x_{ij}y_{ij}}}{\sqrt{\sum\limits_{i,j}{x_{ij}^2}\cdot \sum\limits_{i,j}{y_{ij}^2}}}$, $\mathcal{X}=\sum\limits_{i=1}^B \sum\limits_{j=1}^D{{X_{\text{in}}}_{\left[ i,:,:,j \right]}}\in \mathbb{R}^{T_{\text{in}} \times N}$ for each batch and $B$ is the batch size.

We use the scalar product to calculate the correlation between nodes in the input data $X_{\text{in}}$, and $\cos{\left< \cdot \right>}$ to measure the similarity between graphs and input data, so as to select the most suitable graph for training.

\begin{figure}[h]
	\centering
	\subfigure[The left figure shows that in DCRNN, the information transfer between node $u$ and node $z$ requires three-step diffusion to realize.]{
	\begin{minipage}[b]{.45\linewidth}
	\centering
	\includegraphics[width=0.7\columnwidth]{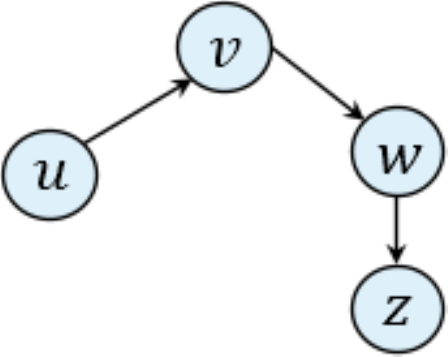}
	\end{minipage}
	}
	\ \ \ \ \ 
	\subfigure[The right picture is the adjacency matrix produced by our graph structure learning module. It learns the associations between all nodes, so there is no need for multi-step diffusion.]{
	\begin{minipage}[b]{.45\linewidth}
	\centering
	\includegraphics[width=0.7\columnwidth]{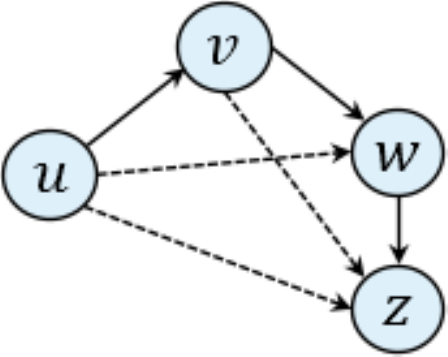}
	\end{minipage}
	}
	\caption{Information Transfer Patterns.}
	\label{DCL}
	\end{figure}

\subsubsection{Diffusion Convolutional Recurrent Network}
Due to its specific design for directed graphs, we adopt DCRNN~\cite{DCRNN} as our forecasting module to capture temporal trends. As is shown in Figure~\ref{DCL}, our graph learning module learns the adjacency relationship between any two nodes and carries out information transfer. Besides, multi-step diffusion results in node homogeneity, which leads to node feature confusion and parametric training difficulties. For this reason, we only use one-step diffusion convolutional operation, which is defined as
\begin{equation}
	W^Q_A X=\left(\omega_0^Q+\omega_1^Q(D_O^{-1}A)+\omega_2^Q(D_I^{-1}A)\right) X,
\end{equation}
with $D_O$ and $D_I$ being the out-degree and in-degree matrices and $\omega_0^Q$, $\omega_1^Q$, $\omega_2^Q$ being model parameters.

We leverage the recurrent neural networks (RNNs) with Gated Recurrent Units (GRU) to model the temporal dependency:
\begin{equation}
	\begin{aligned}
		&R_t\,=\text{sigmoid}\left(W_A^R(X_t||H_{t-1})+b_R\right),\\
		&C_t\,=\text{tanh}\left(W_A^C(X_t||R_t \odot H_{t-1})+b_C\right),\\
		&U_t\,=\text{sigmoid}\left(W_A^U(X_t||H_{t-1})+b_U\right),\\
		&H_t=U_t \odot H_{t-1}+(1-U_t) \odot C_t,
	\end{aligned}
\end{equation}
where $||$ is concatenation along the feature dimension and $\odot$ represents the element-wise product and $b_R$, $b_C$, $b_U$ are model parameters.

\subsection{Framework of BGSLF}
We present the framework of BGSLF in Figure~\ref{framework}. It consists of a graph structure learning module and a temporal forecasting module. The graph structure learning module consists of Multi-Graph Generation Network (MGN) and Smooth Sparse Unit (SSU). The temporal forecasting module contains the Graph Selection Module and the Diffusion Convolutional Recurrent Network. By choosing the appropriate graph during training and testing, the forecasting accuracy can be significantly improved. 
Different from previous models like GTS, AGCRN, and MTGNN, our model can generate a specified number of graphs based on the training MTS data in the graph structure learning module. Therefore, we need to select the most appropriate graph for the temporal forecasting module. Furthermore, it addresses the problem of poor flexibility caused by the global graph and the issue of poor computational efficiency caused by building one graph for each batch. At the same time, we propose SSU with two major parts, the sparsification coefficient part to control the sparse degree and the gradient redefinition technique part to accelerate convergence.

\section{Experiments}

\begin{table}
\centering
\topcaption{Dataset statistics.}
\begin{tabular}{l|ccccc}  
				\toprule         
				Datasets&\ \# Samples\ &\ \# Nodes&\ Sample Rate&\ Input Length&\ Output Length \cr
				\midrule 
				PEMS04&16,969&307&5 minutes&12&12 \cr
				PEMS08&17,833&170&5 minutes&12&12 \cr
				METR-LA&34,272&207&5 minutes&12&12 \cr
				Solar-Energy&52,560&137&10 minutes&12&12 \cr
				\bottomrule 
			\end{tabular}

\label{tab:dataset}
\end{table}

We verify BGSLF on four public multivariate time series datasets, PEMS04, PEMS08, METR-LA, and Solar-Energy. PEMS04 is collected by Caltrans Performance Measurement System (PEMS) and released in ASTGCN\cite{ASTGCN} consisting of average speed, traffic volume in San Francisco Bay Area. Time span is from January to February in 2018. Similar as PeMS04, PEMS08 consists of average speed,
traffic volume collected by PeMS in San Bernardino from July to August in 2016. METR-LA contains average traffic speed measured by 207 sensors on the highway of Los Angles Country ranging from Mar 2012 to Jun 2012. Solar-Energy contains the solar power output from 137 PV plants in Alabama State in 2007. Z-score normalization is applied to inputs. We adopt the same data pre-processing procedures as in~\cite{DCRNN}, and the datasets are split in chronological order with 70$\%$ for training, 10$\%$ for validation, and 20$\%$ for testing. We tune the hyperparameters on the validation data by grid search for BGSLF. We adopt the Adam optimizer, and the number of training epochs is set as 200. Detailed data statistics are provided in Table~\ref{tab:dataset}.

\subsection{Baselines}
Select the traditional and leading performance models. We compare BGSLF with the following models.
\begin{itemize}
\item HA Historical average, which models the traffic flow as a periodic process and uses the weighted average of previous periods as the prediction.   
\item VAR Vector Auto-Regression~\cite{VAR}. 
\item FC-LSTM Recurrent neural network with fully connected LSTM hidden units~\cite{FC-LSTM}. 
\item ASTGCN Attention-based spatio-temporal graph convolutional network, which further integrates spatial and temporal attention mechanisms to STGCN for capturing dynamic spatial and temporal patterns. We take its recent components to ensure the fairness of comparison~\cite{ASTGCN}
\item AGCRN Adaptive Graph Convolutional Recurrent Network. An adaptive graph convolutional network, which designs two adaptive modules for enhancing graph convolutional network with new capabilities~\cite{AGCRN}. 
\item DCRNN Diffusion convolutional recurrent neural network, which incorporates diffusion graph convolution with recurrent neural network in an encoder-decoder manner~\cite{DCRNN}.
\item MTGNN Multivariate time series forecasting with graph neural networks, which uses external features to generate self-adaptive graphs for downstream forecasting module~\cite{MTGNN}.
\item GTS Graph for time series, which aims to jointly learn a latent graph in the time series and use it for MTS forecasting~\cite{GTS}.
\end{itemize}
\begin{table}
\centering
\topcaption{Performance of BGSLF and baselines on four real-world datasets.}
\begin{tabular}{clccc|ccc|ccc}  
				\toprule

				\multirow{2}{*}{Data}&\multirow{2}{*}{ Models}&
				\multicolumn{3}{c}{Horizon 3}&\multicolumn{3}{c}{Horizon 6}&\multicolumn{3}{c}{Horizon 12}\cr 
				
				\cmidrule(lr){3-11}   	
				&&MAE&RMSE&MAPE  &MAE&RMSE&MAPE&MAE&RMSE&MAPE\cr
								\midrule
				\multirow{9}{*}{\rotatebox{90}{PEMS04}}
				&HA&24.50&39.83&16.58$\%$&24.50&39.83&16.58$\%$&24.50&39.83&16.58$\%$\cr
				&VAR&20.85&32.54&15.03$\%$&22.33&34.46&16.24$\%$&25.16&38.11&18.79$\%$\cr
				
				&FC-LSTM&22.33&34.09&18.91$\%$&25.87&39.27&19.94$\%$&34.09&50.27&25.08$\%$\cr
				
				&ASTGCN&19.74&31.50&12.95$\%$&21.49&34.32&13.88$\%$&25.74&40.73&16.22$\%$\cr
				
				&DCRNN&18.58&29.78&13.56$\%$&19.84&31.76&14.71$\%$&22.14&34.61&17.46$\%$\cr
				
				&MTGNN&23.40&33.25&31.75$\%$&24.15&34.53&32.47$\%$&25.99&37.58&32.34$\%$\cr
				
				&AGCRN&18.84&30.74&\bf{12.43}$\%$&19.53&31.92&\bf{12.92}$\%$&21.00&34.28&\bf{13.73}$\%$\cr
				&GTS&18.90&29.81&12.80$\%$&19.58&\bf{31.22}&13.48$\%$&20.96&\bf{33.30}&14.61$\%$\cr
				
				&BGSLF (ours)&\bf{18.49}&\bf{29.43}&12.68$\%$&\bf{19.42}&31.33&13.27$\%$&\bf{20.81}&33.51&14.29$\%$\cr
								\midrule
				\multirow{9}{*}{\rotatebox{90}{PEMS08}}
			    &HA&21.19&36.64&13.79$\%$&21.19&36.64&13.79$\%$&21.19&36.64&13.79$\%$\cr
			    
				&VAR&16.56&24.91&10.62$\%$&19.16&28.58&12.34$\%$&23.22&33.93&15.68$\%$\cr
				
				&FC-LSTM&17.43&26.87&13.64$\%$&20.82&32.28&12.59$\%$&27.61&41.52&17.75$\%$\cr
				
				&ASTGCN&16.03&24.75&10.26$\%$&17.83&27.58&11.20$\%$&21.58&32.80&12.97$\%$\cr
				
				&DCRNN&14.50&\bf{22.48}&9.76$\%$&15.55&24.38&10.45$\%$&17.58&27.30&11.62$\%$\cr
				
				&MTGNN&16.48&24.13&22.84$\%$&17.86&26.22&24.32$\%$&20.41&29.39&28.33$\%$\cr
				
				&AGCRN&15.33&24.06&\bf{9.72}$\%$&16.51&26.11&10.31$\%$&19.01&29.94&11.71$\%$\cr
				&GTS&14.94&23.09&9.77$\%$&15.66&24.33&10.53$\%$&17.07&27.03&11.71$\%$\cr
				
				&BGSLF (ours)&\bf{14.46}&22.58&9.96$\%$&\bf{15.02}&\bf{23.88}&\bf{10.00}$\%$&\bf{16.46}&\bf{26.46}&\bf{10.96}$\%$\cr
				
				\midrule 
    				\multirow{9}{*}{\rotatebox{90}{METR-LA}}
				&HA&4.15&7.77&12.90$\%$&4.15&7.77&12.90$\%$&4.15&7.77&12.90$\%$\cr
				&VAR&4.42&7.89&10.20$\%$&5.41&9.13&12.70$\%$&6.52&10.11&15.80$\%$\cr
				
				&FC-LSTM&3.44&6.30&9.60$\%$&3.77&7.23&10.90$\%$&4.37&8.69&13.20$\%$\cr
				
				&ASTGCN&3.01&5.85&8.16$\%$&3.53&7.14&10.16$\%$&4.25&8.60&12.80$\%$\cr
				
				&DCRNN&2.77&5.38&7.30$\%$&3.15&6.45&8.80$\%$&3.60&7.60&10.50$\%$\cr
				
				&MTGNN&2.69&5.18&6.86$\%$&3.05&6.17&8.19$\%$&3.49&7.23&9.87$\%$\cr
				
				&AGCRN&3.70&9.58&7.93$\%$&4.77&12.15&9.64$\%$&6.12&15.13&11.67$\%$\cr
				&GTS&2.64&\bf{4.95}&6.80$\%$&3.01&\bf{5.85}&8.20$\%$&3.41&\bf{6.74}&9.90$\%$\cr
				
				&BGSLF (ours)&\bf{2.59}&5.09&\bf{6.68$\%$}&\bf{2.97}&6.11&\bf{8.02$\%$}&\bf{3.38}&7.16&\bf{9.44$\%$}\cr

				\bottomrule 
			\end{tabular}
\begin{tabular}{clcc|cc|cc}  
				\toprule

				\multirow{2}{*}{Data}&\multirow{2}{*}{Models}&
				\multicolumn{2}{c}{Horizon 3}&\multicolumn{2}{c}{Horizon 6}&\multicolumn{2}{c}{Horizon 12}\cr 
				
				\cmidrule(lr){3-8}   	
				&&MAE&RMSE&MAE&RMSE&MAE&RMSE\cr
				
				\midrule 
				
				\multirow{9}{*}{\rotatebox{90}{Solar-Energy}}		&HA&6.11&8.74&6.11&8.74&6.11&8.74\cr
				&VAR&0.79&\bf{1.54}&2.55&3.92&4.17&6.12\cr
				
				&FC-LSTM&0.70&1.88&1.17&2.90&1.98&4.74\cr
				
				&ASTGCN&0.65&1.70&0.95&2.47&1.69&4.12\cr
				
				&DCRNN&0.58&1.60&0.84&2.24&1.37&3.64\cr
				
				&MTGNN&1.36&2.46&1.95&3.33&2.91&4.66\cr
				
				&AGCRN&0.58&1.60&0.84&2.24&1.40&3.67\cr
	        	&GTS&0.58&1.66&0.85&2.33&1.37&3.67\cr
				
				&BGSLF (ours)&\bf{0.55}&1.56&\bf{0.80}&\bf{2.16}&\bf{1.31}&\bf{3.53}\cr
				
				\bottomrule 
			\end{tabular}

\label{tab:performance1}
\end{table}
\subsection{Experimental Setups}
We implement our experiments on the platform PyTorch using 4 NVIDIA GeForce RTX 3090 GPUs. The hyper-parameters period $P$ and graphs of BGSLF are set to 288 and 2, respectively. The grid search strategy is executed to choose other hyper-parameters on validation. All of these methods are evaluated with three common metrics: mean absolute error (MAE), root mean square error (RMSE) and mean absolute percentage error (MAPE). Due to the non-uniform distribution of the solar power output of PV plants in spatial and temporal domains, there are many zeros in Solar-Energy. Hence, we only adopt MAE and RMSE in this dataset. We choose to use mean absolute error (MAE) as the training objective of BGSLF. Missing values are excluded both from training and testing. All the tests use 12 observed data points to forecast multivariate time series in the next 3, 6, and 12 steps. The initial learning rate is 3e-3 
with a decay rate of 0.1 per 6 epochs, and the minimum learning rate is 3e-5. Since DCRNN and ASTGCN require a predefined graph, the Solar-Energy dataset does not have one. Therefore, we apply the training multivariate time series to construct a $k$NN graph as the predefined graph structure.

\subsection{Experimental Results}
Table~\ref{tab:performance1} compares the performances of BGSLF and baseline models for Horizon 3, 6, and 12 ahead forecasting on PEMS04, PEMS08, METR-LA, and Solar-Energy datasets. On all four datasets collected at multiple locations and with different sampling rates, our proposed model achieves the start-of-the-art performance whether long-term or short-term, which demonstrates the effectiveness of our proposed models. It outperforms traditional temporal models including HA, VAR, and FC-LSTM by a large margin. Methods like AGCRN and MTGNN apply random initialization to initialize graph structures that lack latent spatial associations among multivariate time series and may fail to capture critical dependencies between nodes, resulting in performance degradation. In addition, the study~\cite{Study} also pointed out that the temporal forecasting part of MTGNN is not sensitive to the graph structure learning module, which means that the learned spatial relationship is insufficient. With respect to the second-best model GTS, we can observe that BGSLF achieves small improvement on METR-LA. However, it can be seen in Table~\ref{tab:parameter} that the number of parameters of our proposed model is nearly 190 times less than that of the GTS model, which saves much memory and dramatically speeds up the training and inference speed. In addition, from the performance of DCRNN and ASTGCN on Solar Energy, it is also an effective method to construct a $k$NN graph in the absence of a predefined graph structure. The design of multi-graph generation network and graph selection module will help our model to dynamically capture the spatial correlation between nodes and balance efficiency and flexibility. 
\begin{table}
\centering
\topcaption{Trainable Parameters of different graph structure learning based spatial-temporal models on METR-LA and PEMS04 when achieving the best results.}
\begin{tabular}{clc}  
				\toprule
				Data&Models&Parameters \cr
				
				\midrule 
				
				\multirow{4}{*}{\shortstack{METR-\\LA}}
				
				&MTGNN&405,452\cr
				
				&AGCRN&747,810\cr
				
	        	&GTS&38,478,291\cr
				
				&BGSLF&\bf{202,266}\cr
				
				\bottomrule 
			\end{tabular}
\ \ \ \ \ \ \ \ \ \ \ \ \ \ 
\begin{tabular}{clc}  
				\toprule
				Data&Models&Parameters \cr
				
				\midrule 
				
				\multirow{4}{*}{\shortstack{PEMS-\\04}}
				
				&MTGNN&549,100\cr
				
				&AGCRN&748,810\cr
				
	        	&GTS&19,125,459\cr
				
				&BGSLF&\bf{229,502}\cr
				
				\bottomrule 
			\end{tabular}
\label{tab:parameter}
\end{table}
\subsection{Effect of the SSU module}
Detailed proof of SSU validity is given in Appendix~\ref{effect}. In Table~\ref{tab:activation}, we replace the SSU with different functions and compare the results. The experimental results prove the correctness of our idea. Firstly, the expressive power of the continuous matrix is indeed better than that of the discrete matrix, which is why the effect obtained by continuous function significantly outweighs Gumbel-softmax. In addition, due to the unique smoothness design and gradient redefinition method, the effect of SSU is better than other common activation functions in this task. The dependencies among different detectors learned by the graph structure learning module are visualized in Figure~\ref{heat}. The upper heat maps in Figure~\ref{heat} show the initial graphs obtained when the model starts training, and the lower represent the graph structure obtained when the model finishes training. It can be seen that our model successfully learned the sparse relationship between detectors, which corresponds to
the sparse spatial correlations in the real world.

\begin{table}
\centering
\topcaption{Performance of different functions on METR-LA.}
\begin{tabular}{llccc|ccc|ccc}  
				\toprule

				\multirow{2}{*}{Data}&\multirow{2}{*}{Function}&
				\multicolumn{3}{c}{Horizon 3}&\multicolumn{3}{c}{Horizon 6}&\multicolumn{3}{c}{Horizon 12}\cr 
				
				\cmidrule(lr){3-11}   	
				&&MAE&RMSE&MAPE  &MAE&RMSE&MAPE&MAE&RMSE&MAPE\cr
				
				\midrule 
    			\multirow{4}{*}{\shortstack{METR-\\LA}}
				
				&Sigmoid&2.68&5.30&6.94$\%$&3.07&6.34&8.32$\%$&3.48&7.41&9.78$\%$\cr
				
				&Tanh&2.60&5.12&6.79$\%$&2.99&6.18&8.38$\%$&3.40&7.23&9.99$\%$\cr
				&Gumbel-softmax&2.85&5.72&7.54$\%$&3.41&7.02&9.72$\%$&4.15&8.52&12.78$\%$\cr
				
				&SSU&\bf{2.59}&\bf{5.09}&\bf{6.68$\%$}&\bf{2.97}&\bf{6.11}&\bf{8.02$\%$}&\bf{3.38}&\bf{7.16}&\bf{9.44$\%$}\cr
				\bottomrule 
			\end{tabular}

\label{tab:activation}
\end{table}

\begin{figure}[H]
	\centering
	\subfigure[The initial and final states of the first graph adjacency matrix.]{
	\begin{minipage}[b]{.45\linewidth}
	\centering
	\includegraphics[width=1\columnwidth]{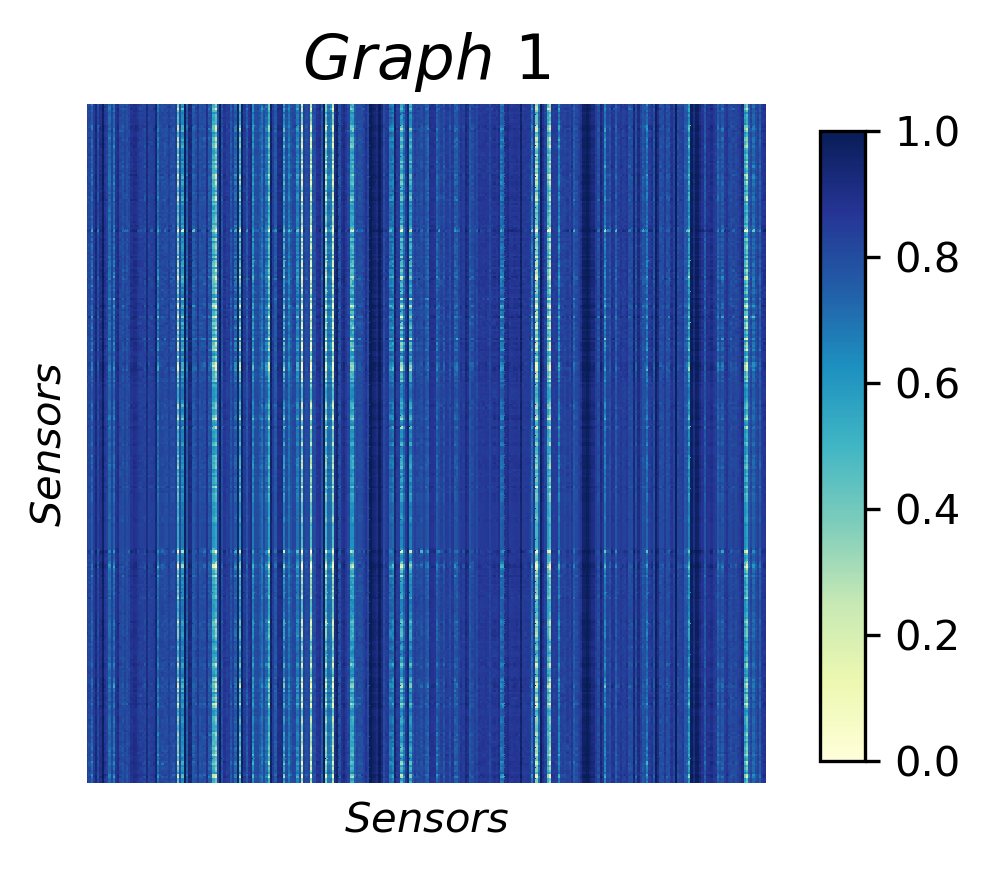}\\
	\includegraphics[width=1\columnwidth]{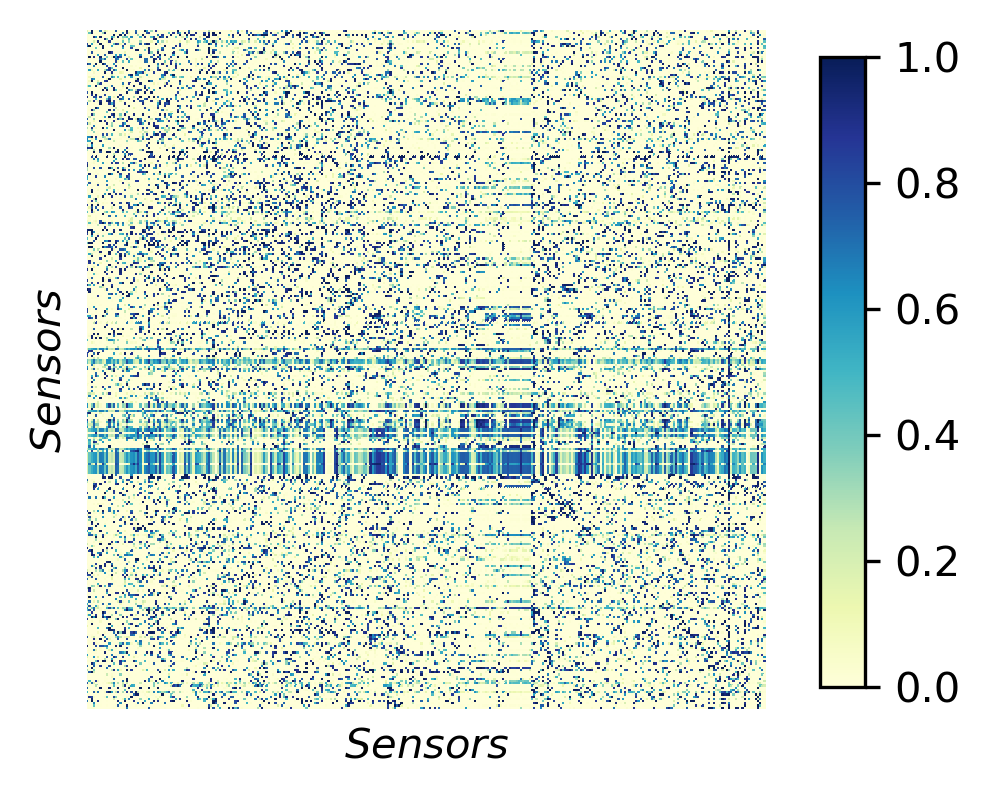}
	\end{minipage}
	}
	\ \ \ \ \ 
	\subfigure[The initial and final states of the second graph adjacency matrix.]{
	\begin{minipage}[b]{.45\linewidth}
	\centering
	\includegraphics[width=1\columnwidth]{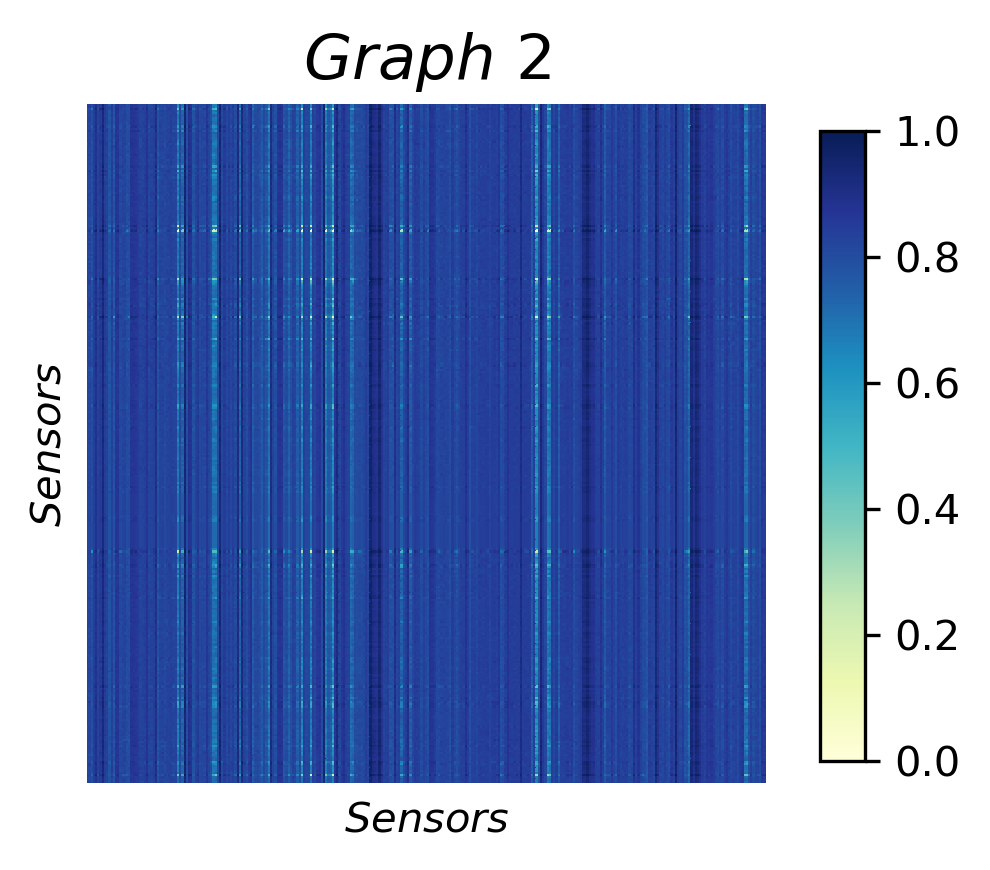}\\
	\includegraphics[width=1\columnwidth]{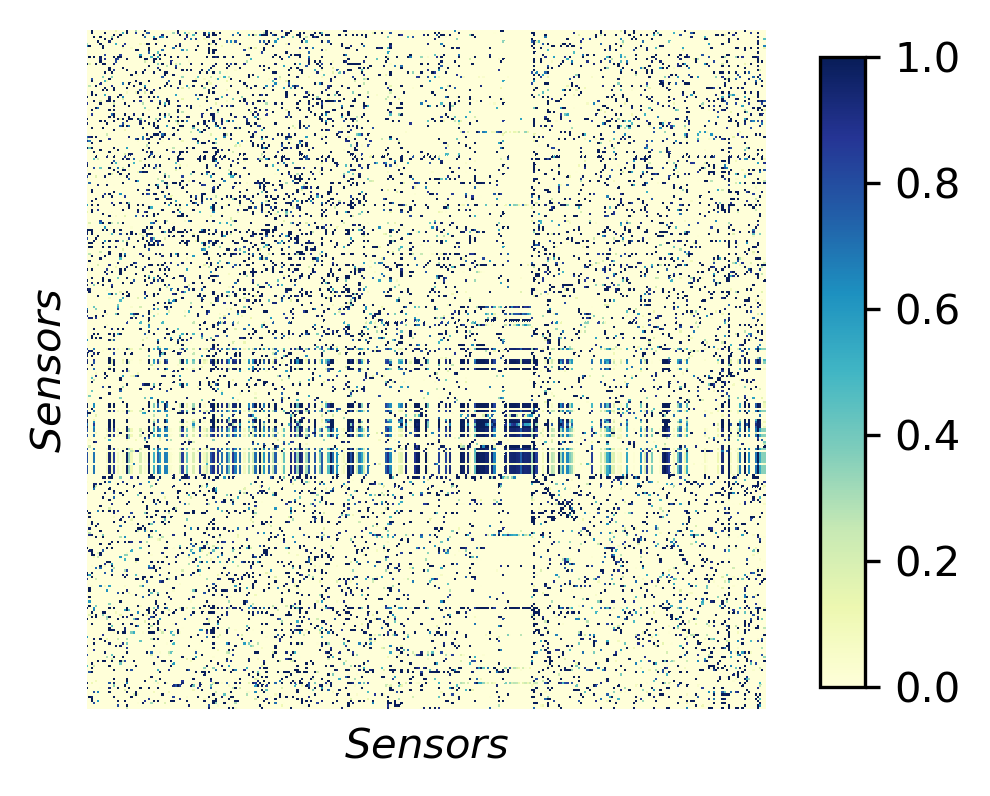}
	\end{minipage}
	}
    \caption{ The adjacent matrices obtained from the graph structure learning module on the PEMS04 dataset.}
    \label{heat}
\end{figure}

\section{Conclusion}
In this paper, we propose a novel model that joins graph structure learning and forecasting. Our model highlights three critical issues with previous models and provides concise and practical solutions. Our research emphasizes that the training multivariate time series can be applied to generate a specified number of valid graph structures, and the optimal spatial structure can be selected by computing the similarity of each input multivariate time series. With our unique design of multi-graph generation network and graph selection module, our model is well balanced between efficiency and flexibility.  Experiments were conducted on four real-world datasets to demonstrate the superiority of our proposed model. The well-trained embeddings and learned graphs could also be potentially applied to other tasks.

	\appendix

	\section{Description of SSU module}\label{effect}
	    \begin{figure}[h]
	\centering
	\subfigure[The smooth function $\varphi(x), \alpha = 1$.]{
	\includegraphics[width=0.47\columnwidth]{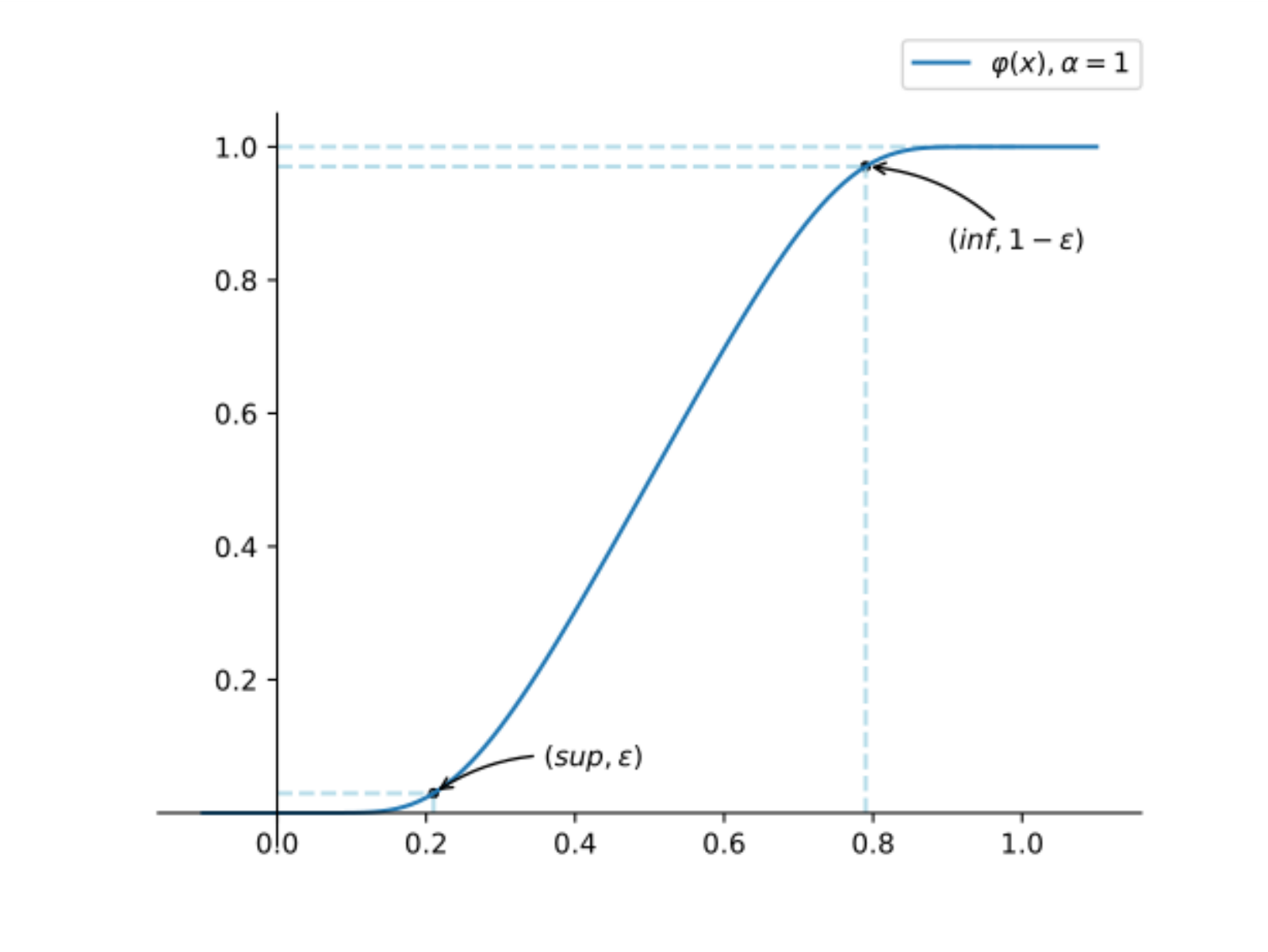}
	\label{SSU}
	}
	\subfigure[The curves of different $\alpha$ values.]{
	\includegraphics[width=0.47\columnwidth]{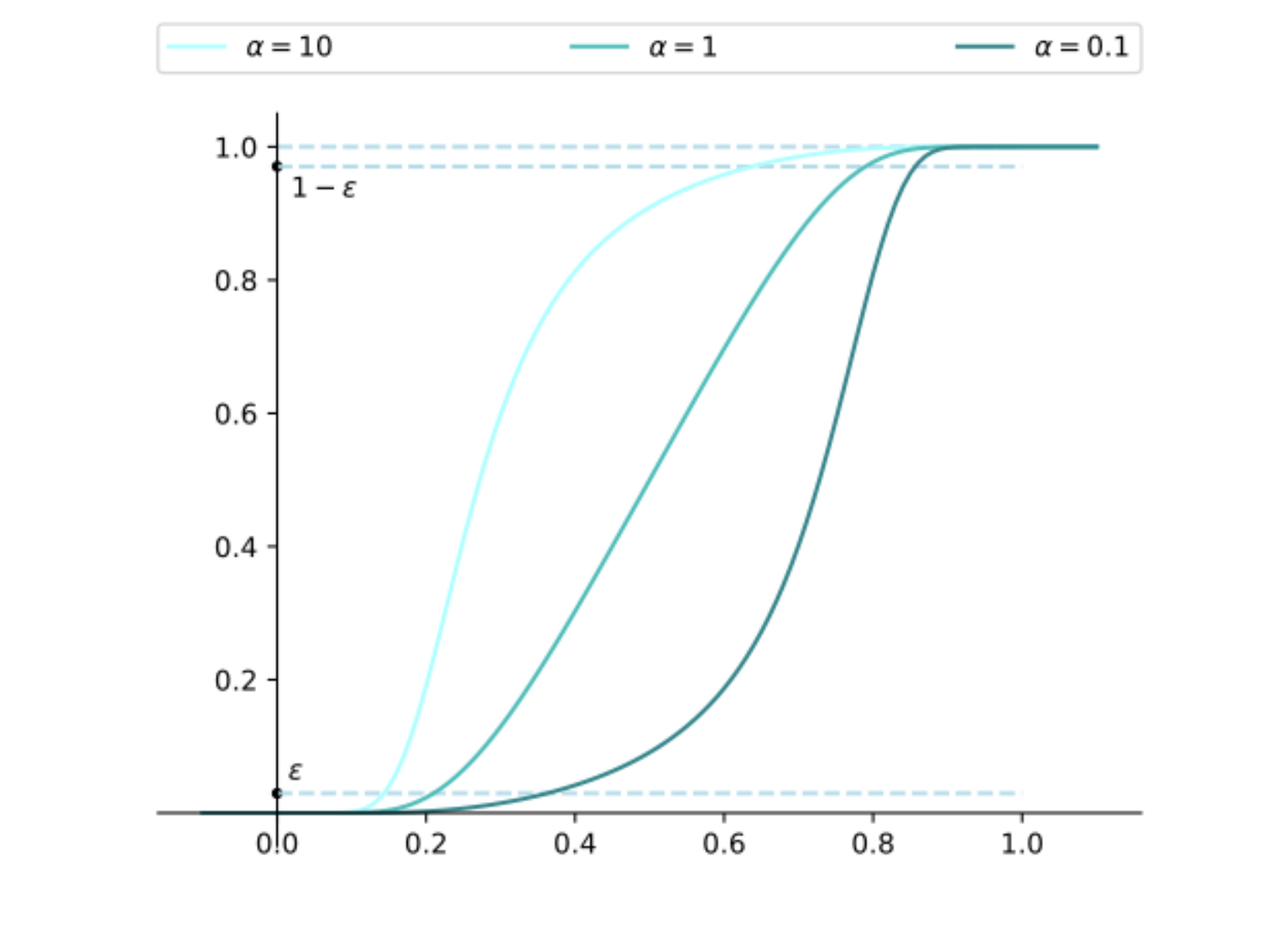}
	\label{alpha}}
	\caption{Basic curves of SSU.}
	\end{figure} 
	\subsection{Sparsification coefficient}
	As we defined in context, let
	\begin{equation}
	\begin{aligned}
     &f(x)=\left\{
    \begin{aligned}
    &e^{-\frac{1}{x}} &(x>0), \\
    &0 &(x \leq 0), 
    \end{aligned}
    \right.
    \\
    &\varphi(x)=\dfrac{\alpha f(x)}{\alpha f(x)+f(1-x)}\ (\alpha \in \mathbb{R}_+),
    \end{aligned}
    \end{equation}

	where parameter $\alpha$ is the sparsification coefficient. It can determine the shape of the curve $\varphi$ and the sparse degree of the generated adjacency matrices.
	The sparsification effect of SSU is described below.
	
	It is obvious that $\varphi(x) \equiv 0$, $\varphi'(x) \equiv 0$ for $x \leq 0$; $\varphi(x) \equiv 1$, $\varphi'(x) \equiv 0$ for $x \geq 1$. So we just consider $0<x<1$, and let $t = \dfrac{f(1-x)}{f(x)}=e^{\frac{1}{x}-\frac{1}{1-x}} := g(x)$. For $g'(x) = e^{\frac{1}{x}-\frac{1}{1-x}}\left[-\dfrac{1}{x^2}-\dfrac{1}{(1-x)^2}\right]<0$, $g(x)$ decreases strictly monotonically on $(0,1)$. Thus $g$ is a bijection and has an inverse function $g^{-1}$. \\
	For $\varphi(x)<\varepsilon$, i.e.$\dfrac{\alpha}{\alpha+t} < \varepsilon$,\\ $t>\alpha\left(\dfrac{1}{\varepsilon}-1\right) \iff x<g^{-1}\left(\alpha\left(\dfrac{1}{\varepsilon}-1\right)\right)\triangleq\text{sup}$.\\
	For $\varphi(x)>1-\varepsilon$, i.e.$\dfrac{\alpha}{\alpha+t} > 1 - \varepsilon$,\\ $t<\alpha\left(\dfrac{1}{1-\varepsilon}-1\right) \iff x>g^{-1}\left(\alpha\left(\dfrac{1}{1-\varepsilon}-1\right)\right)\triangleq\text{inf}$.\\
	
	In Figure~\ref{alpha}, fixing $\varepsilon$, as $\alpha$ decreases, sup, inf increase and the length of interval  $\varphi^{-1}((0,\varepsilon))=(0,\text{sup})$ increase, and vice versa. If we consider the elements in the adjacency matrix $A=(A_{ij})_{n \times n}$ have a uniform distribution in $[0,1]$, then as $\alpha$ decreases, the probability of $a_{ij}$ falling into $(0,\text{sup})$ and $A$ being sparse increases. Therefore, we get the conclusion that $\alpha$ can control the sparsification effect of SSU.
	\subsection{Gradient redefinition}
	In our experiments, as $x$ approaches 0 and 1, the gradient approaches 0 rapidly, which leads to the vanishing gradient problem. In fact, it is extensively difficult to train the adjacency matrix values to zero and achieve the sparsification effect. Therefore, we redefine the gradient as 1 in intervals $(0,\text{sup})$ and $(\text{inf},1)$ to accelerate convergence making the activation value fall into $\{0,1\}$ or $(\varphi(\text{sup}),\varphi(\text{inf}))$ faster.

\bibliographystyle{splncs04}
\bibliography{ECML}
%
%
%
%

\end{document}